\documentclass{article} 
\usepackage{times}
\usepackage{mathptmx}

\usepackage{amsmath,amsfonts,bm}

\def\eqref#1{equation~\ref{#1}}

\def\1{\bm{1}}

\DeclareMathAlphabet{\mathsfit}{\encodingdefault}{\sfdefault}{m}{sl}
\SetMathAlphabet{\mathsfit}{bold}{\encodingdefault}{\sfdefault}{bx}{n}

\usepackage{hyperref}
\usepackage{url}
\usepackage{graphicx}
\usepackage{booktabs}
\usepackage{multirow}
\usepackage[table]{xcolor}
\usepackage{amsthm}
\usepackage{natbib}
\newtheorem{proposition}{Proposition}

\title{Dual-Axis Policy Optimization for LLM Agents: Bayesian Feedback Attribution and Trajectory Mass Normalization}

\author{
Yingxuan Zhuang$^{1,*}$ \quad
Binhe Yu$^{1,*}$ \quad
Jingxiao Yang$^{1,*}$ \quad
Ruopei Sun$^{2,*}$ \\
Ziting li$^{3}$\quad
Cheng Tan $^{1}$\quad
Xuhong Zhang $^{1}$\quad
Jianwei Yin $^{1}$\\
Jintao Chen$^{1}$\\
$^{1}$Zhejiang University \\
$^{2}$University of Science and Technology of China\\
$^{3}$University of New South Wales\\
$^{4}$Shanghai Artificial Intelligence Laboratory \\
$^{*}$Equal contribution
}

\makeatletter
\@ifundefined{c@tmntheorem}{%
}{}
\makeatletter
\@ifundefined{c@bayestheorem}{}{}
\makeatother

\makeatother

\begin{document}

\maketitle
\begin{abstract}
Reinforcement learning for LLM agents involves two distinct optimization
dimensions: how environment feedback is exploited within a trajectory, and
how complete trajectories are aggregated across a batch.
We formulate these dimensions as \textbf{Intra-Trajectory Feedback Attribution}
and \textbf{Inter-Trajectory Objective Aggregation}, and introduce
\textbf{BATON} (\textbf{B}ayesian \textbf{A}ttribution and
\textbf{T}rajectory \textbf{O}bjective \textbf{N}ormalization), a dual-axis
policy optimization framework.
BATON instantiates the first axis with Bayesian Feedback Attribution, which
constructs a feedback-conditioned posterior over sampled actions, and the
second with Trajectory Mass Normalization (TMN), which assigns equal
optimization mass to complete trajectories.
Experiments with GRPO and GiGPO on ALFWorld, WebShop, and SearchQA show that both axes provide independent gains and that
their combination consistently achieves the strongest overall performance
across model scales.
\end{abstract}

\section{Introduction}
\label{sec:introduction}

Reinforcement learning is increasingly used to train LLM agents that reason,
call tools, and act through multi-turn interaction rather than produce a single
response. Prompted and tool-augmented agents already operate in embodied, web,
software, and conversational environments
\citep{yao2023react,schick2023toolformer,shridhar2021alfworld,
yao2022webshop,wang2022scienceworld,yang2024sweagent}, while recent benchmarks
expose increasingly long and heterogeneous interaction trajectories
\citep{liu2023agentbench,ma2024agentboard,zhou2023webarena,
deng2023mind2web,drouin2024workarena,lu2024weblinx,yao2024taubench}.
Beyond prompting and supervised trajectory tuning
\citep{zeng2023agenttuning}, recent work increasingly optimizes agent policies
directly with reinforcement learning
\citep{qi2024webrl,wang2025ragen,gigpo,zeng2025turncredit,
luo2025agentlightning}.

Policy optimization for LLMs has rapidly evolved from PPO-based RLHF
\citep{schulman2017ppo,ouyang2022training} toward critic-free and group-based
objectives such as RLOO and GRPO \citep{ahmadian2024back,grpo}, together with
refinements to clipping, reduction, and sequence-level optimization
\citep{yu2025dapo,liu2025drgrpo,hu2025reinforcepp,zheng2025gspo}.
In multi-turn agents, however, optimization acts on structured trajectories in
which each action is followed by an environment observation. This exposes two
conceptually distinct questions: how observed feedback should affect learning
signals \emph{within} a trajectory, and how complete trajectories should
contribute to the objective \emph{across} a batch. We call these dimensions
\textbf{Intra-Trajectory Feedback Attribution} and
\textbf{Inter-Trajectory Objective Aggregation}.

This distinction becomes especially visible in variable-length interaction.
As shown in Figure~\ref{fig:motivation}, unsuccessful trajectories are often substantially longer
than successful ones. 
Inspection of prolonged failures reveals repeated
ineffective actions, misread state changes, and actions taken before required
constraints are satisfied---behaviors consistent with insufficient exploitation
of environment feedback. 
Meanwhile, global valid-token averaging creates a
separate optimization effect: trajectories containing more valid policy tokens
receive proportionally larger scalar mass. In our training data, the longest
25\% of trajectories account for 60.3\% of the aggregation mass in ALFWorld
and 42.2\% in WebShop. Thus, prolonged failures expose two distinct learning
challenges: feedback may be insufficiently exploited locally, while long
trajectories may simultaneously dominate aggregation globally. The measured
aggregation imbalance does not by itself establish the behavioral cause of
failure; it motivates treating the two dimensions separately.

\begin{figure}[!t]
\centering
\includegraphics[
    width=\linewidth,
    trim=220 100 160 135,
    clip
]{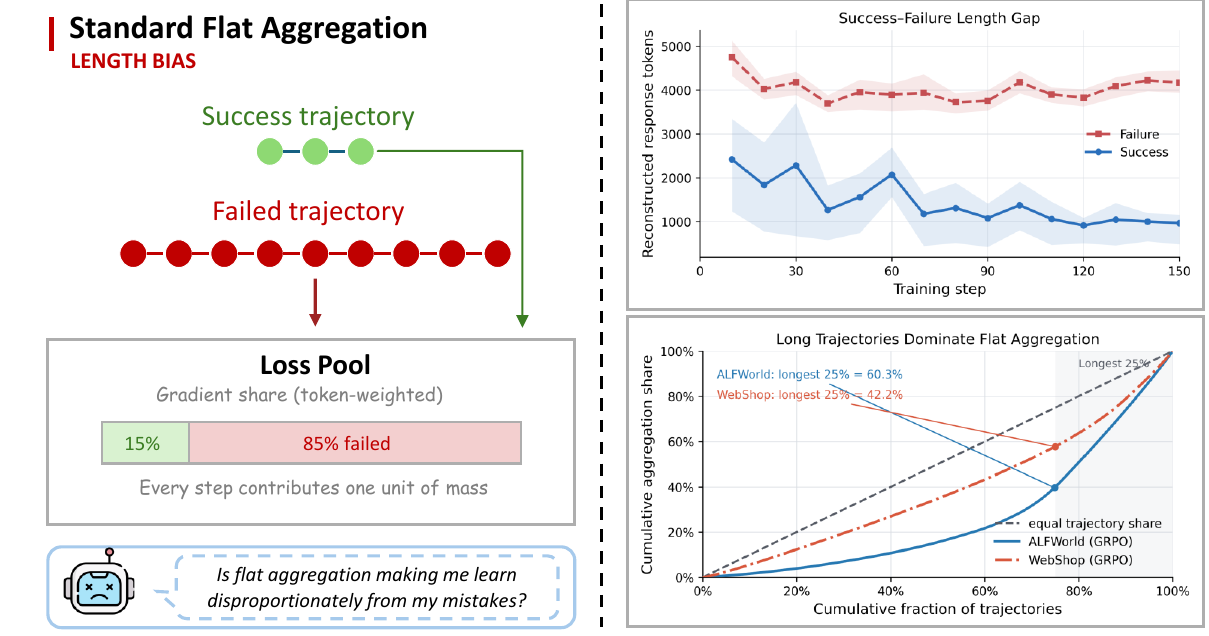}
\caption{
Motivation of BATON.
Flat token-level aggregation gives longer trajectories disproportionate
optimization mass, motivating separate treatment of feedback attribution and
trajectory aggregation.
}
\label{fig:motivation}
\end{figure}

We introduce \textbf{BATON} (\textbf{B}ayesian \textbf{A}ttribution and
\textbf{T}rajectory \textbf{O}bjective \textbf{N}ormalization), a dual-axis
policy optimization framework. For Intra-Trajectory Feedback Attribution,
BATON uses Bayesian Feedback Attribution. Unlike verbal reflection or generic
self-correction
\citep{shinn2023reflexion,madaan2023selfrefine,gou2024critic,
zhou2023lats,wang2023voyager}, it compares actions sampled under the same
pre-action context by combining their empirical frequencies with the likelihood
of the observed next environment response. The resulting feedback-conditioned posterior determines a normalized
attribution weight that reallocates the host learner's policy signal
across decisions within each trajectory; we further show that
its expected attribution gain is governed by the distinguishability of
action-conditioned feedback distributions. For Inter-Trajectory Objective
Aggregation, Trajectory Mass Normalization (TMN) averages valid learning terms
within each complete trajectory before averaging trajectories, replacing
token-proportional outer mass with a uniform $1/B$ share. Both components
leave the host learner's reward, advantage estimator, clipping rule, and rollout
procedure unchanged.

We evaluate BATON with GRPO and GiGPO on ALFWorld, WebShop, and
search-augmented question answering, where multi-turn search policies have
recently been optimized with outcome- and step-level reinforcement learning
\citep{jin2025searchr1,sun2025zerosearch,zheng2025stepsearch}.
Across model scales, both axes yield independent gains and their combination is
consistently strongest. Improvements also persist with GiGPO, which already
provides finer-grained within-trajectory credit assignment, and are accompanied
by shorter evaluation trajectories and fewer generated policy tokens.

\section{Preliminaries}
\label{sec:preliminaries}

\paragraph{Problem setup.}
We consider a language-model agent interacting with an environment to complete multi-step tasks.
At decision step $t$, the agent conditions on the interaction history
$h_t=(o_0,a_0,\ldots,o_t)$ and generates an action
$a_t\sim\pi_\theta(\cdot\mid h_t)$, after which the environment returns the next observation $o_{t+1}$.
A complete trajectory $\tau_i$ contains $T_i$ decision steps and receives a task return $R_i$.
Because each textual action may contain multiple policy tokens, we let $n_{i,t}$ denote the number of valid policy tokens at step $t$ and
$N_i=\sum_{t=1}^{T_i}n_{i,t}$ the total number of valid policy tokens in trajectory $\tau_i$.
We optimize over a batch of $B$ complete trajectories.

\paragraph{Group-based policy optimization.}
Recent agent RL methods such as GRPO and GiGPO optimize groups of sampled trajectories without requiring a separate value model.
For GRPO, each trajectory $\tau_i$ is assigned a group-relative advantage $\hat A_i$ based on the returns of sibling rollouts.
Let
$r_{i,t,k}(\theta)
=
\pi_\theta(y_{i,t,k}\mid h_{i,t,k})/
\pi_{\theta_{\rm old}}(y_{i,t,k}\mid h_{i,t,k})$
denote the likelihood ratio of the $k$-th valid policy token.
Under global valid-policy-token averaging, the policy objective is
\begin{equation}
\mathcal L_{\rm GRPO}(\theta)
=
-\frac{1}{S}
\sum_{i=1}^{B}
\sum_{t=1}^{T_i}
\sum_{k=1}^{n_{i,t}}
\min\!\left[
r_{i,t,k}(\theta)\hat A_i,\,
\operatorname{clip}\!\left(
r_{i,t,k}(\theta),1-\epsilon,1+\epsilon
\right)\hat A_i
\right],
\quad
\label{eq:grpo}
\end{equation}
where $S=\sum_{i=1}^{B}N_i$. We omit separately reduced regularization terms for clarity.

\paragraph{Flat trajectory aggregation.}
More generally, let $\ell^{\mathcal A}_{i,t,k}$ denote the local
policy-surrogate loss of a host learner $\mathcal A$, and let
$\bar{\ell}^{\mathcal A}_i
=
N_i^{-1}\sum_{t,k}\ell^{\mathcal A}_{i,t,k}$
be its trajectory-level mean.
A generic trajectory-weighted objective can be written as
\begin{equation}
\mathcal J_{\mathcal A}(\theta;\mathbf q)
=
\sum_{i=1}^{B}q_i\bar{\ell}^{\mathcal A}_i(\theta),
\qquad
\sum_{i=1}^{B}q_i=1,
\qquad
q_i^{\rm flat}
=
\frac{N_i}{\sum_jN_j}.
\label{eq:trajectory_aggregation}
\end{equation}
Thus, global valid-token averaging assigns trajectory $\tau_i$ a scalar
contribution proportional to its number of valid policy tokens.

\section{Method}
\label{sec:method}

We propose \textbf{BATON}, a dual-axis policy optimization framework for
multi-turn LLM agents.
BATON separates policy optimization into two complementary dimensions:
\textbf{Intra-Trajectory Feedback Attribution}, which allocates learning mass
among decisions within a trajectory, and
\textbf{Inter-Trajectory Objective Aggregation}, which determines the scalar
mass assigned to complete trajectories.
We instantiate the first axis with Bayesian Feedback Attribution (BFA) and the
second with Trajectory Mass Normalization (TMN).
Both components modify different levels of the objective and can be combined
with host learners such as GRPO and GiGPO.

\begin{figure*}[t]
\centering
\includegraphics[width=\linewidth]{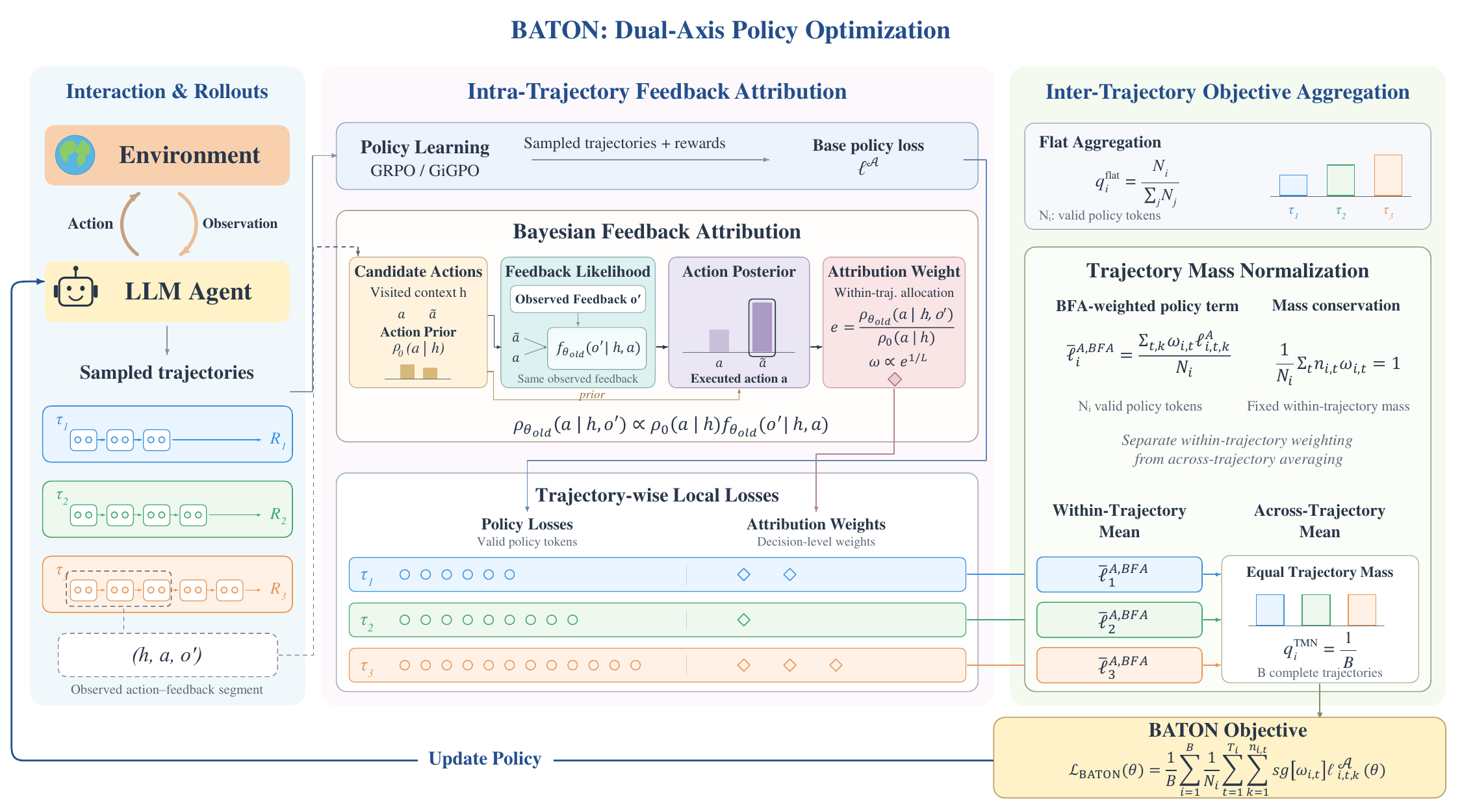}
\caption{
Overview of BATON.
Bayesian Feedback Attribution evaluates feedback compatibility between visited
decisions and observed environment responses, producing normalized
stop-gradient weights within trajectories.
Trajectory Mass Normalization controls the outer aggregation across complete
trajectories.
}
\label{fig:method}
\end{figure*}

\subsection{Dual-Axis Policy Optimization}

For a batch of trajectories
$\mathcal B=\{\tau_i\}_{i=1}^{B}$,
we decompose policy optimization into within-trajectory allocation and
across-trajectory aggregation:

\begin{equation}
\mathcal L(\theta)
=
\sum_{i=1}^{B}
q_i
\frac{1}{N_i}
\sum_{t=1}^{T_i}
\sum_{k=1}^{n_{i,t}}
\omega_{i,t}
\ell^{\mathcal A}_{i,t,k}(\theta),
\label{eq:dual_axis_objective}
\end{equation}

where $q_i$ controls trajectory-level mass and $\omega_{i,t}$ controls the
allocation among decisions.
We require

\begin{equation}
\frac{1}{N_i}
\sum_t n_{i,t}\omega_{i,t}=1 ,
\label{eq:within_mass_constraint}
\end{equation}

so that attribution redistributes mass inside trajectories without changing
their total contribution.

\subsection{Bayesian Feedback Attribution}
\label{sec:method_bayesian_feedback}

BFA estimates how strongly a visited decision is associated with its observed
feedback.
The host learner remains responsible for the signed optimization direction;
BFA only adjusts the relative contribution of different decisions.

For a visited decision $(h,a,o')$, where $h$ is the pre-action context,
$a$ is the executed action, and $o'$ is the observed environment feedback,
BFA samples a counterfactual action

\begin{equation}
\tilde a\sim\pi_{\theta_{\rm old}}(\cdot|h).
\label{eq:counterfactual_sampling}
\end{equation}

The executed and counterfactual actions form a pairwise candidate set.
For each candidate action, BFA evaluates the compatibility with the observed
feedback:

\begin{equation}
\log f_{\theta_{\rm old}}(o'|h,a)
=
\sum_{m\in\mathcal V(o')}
\log
\pi_{\theta_{\rm old}}
(o'_m|h,a,o'_{<m}),
\label{eq:feedback_likelihood}
\end{equation}

where $\mathcal V(o')$ denotes valid feedback-target token positions.
The likelihood measures model-based feedback compatibility rather than causal
effects of counterfactual execution.

\paragraph{Bayesian evidence.}

Under the pairwise construction,
$\rho_0(a|h)=\rho_0(\tilde a|h)=1/2$.
The feedback-conditioned posterior of the executed action is

\begin{equation}
\rho_{\theta_{\rm old}}(a|h,o')
=
\frac{
f_{\theta_{\rm old}}(o'|h,a)
}{
f_{\theta_{\rm old}}(o'|h,a)
+
f_{\theta_{\rm old}}(o'|h,\tilde a)
}.
\label{eq:pairwise_posterior}
\end{equation}

We define the Bayesian evidence ratio

\begin{equation}
e(h,a,o')
=
\frac{
\rho_{\theta_{\rm old}}(a|h,o')
}{
\rho_0(a|h)
}
=
\frac{
2f_{\theta_{\rm old}}(o'|h,a)
}{
f_{\theta_{\rm old}}(o'|h,a)
+
f_{\theta_{\rm old}}(o'|h,\tilde a)
}.
\label{eq:evidence_ratio}
\end{equation}

\begin{proposition}[Bayesian Evidence Interpretation]
\label{prop:evidence_interpretation}

Let

\begin{equation}
m(o'|h)
=
\frac12f_{\theta_{\rm old}}(o'|h,a)
+
\frac12f_{\theta_{\rm old}}(o'|h,\tilde a)
\end{equation}

denote the pairwise feedback mixture.
The Bayesian evidence satisfies

\begin{equation}
e(h,a,o')
=
\frac{
f_{\theta_{\rm old}}(o'|h,a)
}{
m(o'|h)
}.
\end{equation}

Moreover,

\begin{equation}
\mathbb E_{o'\sim f_{\theta_{\rm old}}(\cdot|h,a)}
[
\log e(h,a,o')
]
=
D_{\mathrm{KL}}
(
f_{\theta_{\rm old}}(\cdot|h,a)
\|
m(\cdot|h)
)
\geq 0 .
\end{equation}

\end{proposition}

Proposition~\ref{prop:evidence_interpretation} characterizes the statistical
quantity underlying BFA evidence: larger evidence corresponds to greater
distinguishability between the executed action's feedback distribution and its
policy-sampled alternative.
The proof is provided in Appendix~\ref{app:bfa_theory}.

\paragraph{Length-calibrated attribution.}

For decision $(i,t)$, let

\[
e_{i,t}=e(h_{i,t},a_{i,t},o_{i,t+1}).
\]

Since the feedback likelihood is accumulated over tokens, we use geometric
evidence

\begin{equation}
r_{i,t}=e_{i,t}^{1/L_{i,t}},
\end{equation}

where $L_{i,t}=|\mathcal V(o_{i,t+1})|$.
The normalized attribution weight is

\begin{equation}
\omega_{i,t}
=
\frac{
r_{i,t}
}{
\frac1{N_i}
\sum_{s=1}^{T_i}
n_{i,s}r_{i,s}
}.
\label{eq:bfa_weight}
\end{equation}

This normalization gives

\begin{equation}
\frac1{N_i}
\sum_t n_{i,t}\omega_{i,t}=1,
\label{eq:bfa_mass_conservation}
\end{equation}

preserving each trajectory's total optimization mass.

BFA applies detached weights to the host objective:

\begin{equation}
\bar{\ell}^{\mathcal A,\mathrm{BFA}}_i(\theta)
=
\frac1{N_i}
\sum_{t,k}
\operatorname{sg}[\omega_{i,t}]
\ell^{\mathcal A}_{i,t,k}(\theta).
\label{eq:bfa_policy_loss}
\end{equation}

Therefore, BFA only rescales host learner updates and does not introduce an
independent optimization direction.

\subsection{Trajectory Mass Normalization}
\label{sec:method_tmn}

BFA controls allocation within trajectories, while TMN controls aggregation
across trajectories.

Under global valid-token averaging, the trajectory coefficient is

\begin{equation}
q_i^{\rm flat}
=
\frac{N_i}{\sum_jN_j}.
\label{eq:flat_mass}
\end{equation}

Thus, longer trajectories receive larger objective mass.
TMN instead assigns equal mass to complete trajectories:

\begin{equation}
q_i^{\rm TMN}
=
\frac1B .
\label{eq:tmn_mass}
\end{equation}

Because BFA preserves the within-trajectory mass constraint,
the two components are independent:
BFA reallocates mass inside trajectories, whereas TMN changes only the
outer trajectory aggregation.

\subsection{Unified BATON Objective}
\label{sec:method_unified}

Combining BFA and TMN gives

\begin{equation}
\mathcal L_{\rm BATON}(\theta)
=
\frac1B
\sum_{i=1}^{B}
\frac1{N_i}
\sum_{t=1}^{T_i}
\sum_{k=1}^{n_{i,t}}
\operatorname{sg}[\omega_{i,t}]
\ell^{\mathcal A}_{i,t,k}(\theta).
\label{eq:baton_objective}
\end{equation}

BFA determines the normalized within-trajectory allocation
$\omega_{i,t}$, while TMN determines the uniform outer trajectory mass
$1/B$.
This formulation yields the controlled $2\times2$ decomposition in our
ablation study:
the host learner uses uniform within-trajectory weights and token-proportional
aggregation; BFA modifies only within-trajectory allocation; TMN modifies only
trajectory-level aggregation; BATON combines both.

\begin{table*}[t]
\centering
\caption{
Main results on ALFWorld and WebShop.
We report success rate (\%) for each ALFWorld task domain and the overall
result, and task score and success rate (\%) on WebShop. Within each backbone, bold marks the highest mean in each metric column
separately within the GRPO-based and GiGPO-based method groups.
}
\label{tab:main_results_agent}
\footnotesize
\setlength{\tabcolsep}{1.45mm}
\renewcommand{\arraystretch}{1.08}
\resizebox{\linewidth}{!}{%
\begin{tabular}{lccccccccc}
\toprule
\multirow{2}{*}{Method}
& \multicolumn{7}{c}{ALFWorld} & \multicolumn{2}{c}{WebShop} \\
\cmidrule(lr){2-8}\cmidrule(lr){9-10}
& Pick & Look & Clean & Heat & Cool & Pick2 & All & Score & Succ. \\
\midrule
\multicolumn{10}{l}{\textit{Closed-Source Models}} \\
GPT-4o & 75.5 & 60.5 & 31.6 & 56.5 & 21.9 & 49.4 & 48.1 & 32.1 & 23.5 \\
Gemini-2.5-Pro & 92.5 & 63.7 & 61.9 & 69.3 & 26.2 & 58.9 & 60.4 & 42.2 & 36.1 \\
\midrule
\multicolumn{10}{l}{\textit{Qwen2.5-1.5B-Instruct}} \\
Qwen2.5 & 6.1 & 5.2 & 3.4 & 9.5 & 4.5 & 0.0 & 4.2 & 22.8 & 5.4 \\
ReAct & 17.1 & 20.9 & 15.5 & 6.5 & 7.6 & 2.2 & 12.9 & 40.4 & 11.1 \\
Reflexion & 35.7 & 22.0 & 22.0 & 13.3 & 19.6 & 3.6 & 22.0 & 55.5 & 22.1 \\
PPO (with critic) & $65.1_{\pm 3.3}$ & $40.1_{\pm 7.1}$ & $57.3_{\pm 5.0}$ & $60.3_{\pm 6.4}$ & $46.8_{\pm 4.1}$ & $47.2_{\pm 2.1}$ & $54.5_{\pm 3.0}$ & $73.5_{\pm 3.2}$ & $51.7_{\pm 2.8}$ \\
RLOO & $88.0_{\pm 3.2}$ & $53.2_{\pm 8.4}$ & $70.8_{\pm 6.1}$ & $63.1_{\pm 8.6}$ & $66.0_{\pm 5.7}$ & $57.1_{\pm 4.5}$ & $69.6_{\pm 2.6}$ & $74.2_{\pm 5.4}$ & $51.9_{\pm 6.8}$ \\
GRPO & $85.7_{\pm 1.6}$ & $53.4_{\pm 8.2}$ & $84.7_{\pm 6.6}$ & $77.8_{\pm 8.1}$ & $60.0_{\pm 4.8}$ & $53.3_{\pm 5.7}$ & $72.9_{\pm 3.4}$ & $75.6_{\pm 3.6}$ & $57.1_{\pm 3.6}$ \\
GRPO + EnvRL & $85.9_{\pm 2.3}$ & $53.9_{\pm 6.9}$ & $85.0_{\pm 4.5}$ & $82.0_{\pm 4.7}$ & $76.8_{\pm 3.9}$ & $71.9_{\pm 4.4}$ & $77.4_{\pm 3.8}$ & $83.0_{\pm 2.1}$ & $67.0_{\pm 2.6}$ \\
\rowcolor{blue!8}
GRPO + BATON (Ours) & $\mathbf{91.0}_{\pm 3.6}$ & $\mathbf{80.1}_{\pm 4.9}$ & $\mathbf{91.9}_{\pm 3.4}$ & $\mathbf{83.6}_{\pm 3.8}$ & $\mathbf{83.9}_{\pm 3.5}$ & $\mathbf{78.5}_{\pm 4.7}$ & $\mathbf{85.6}_{\pm 2.8}$  & $\mathbf{85.2}_{\pm 2.0}$ & $\mathbf{68.2}_{\pm 2.8}$ \\
GiGPO & $94.6_{\pm 5.7}$ & $67.2_{\pm 4.8}$ & $94.5_{\pm 3.9}$ & $94.7_{\pm 7.6}$ & $79.4_{\pm 4.9}$ & $76.6_{\pm 5.3}$ & $86.8_{\pm 1.8}$ & $83.7_{\pm 1.9}$ & $67.1_{\pm 4.7}$ \\
GiGPO + EnvRL & $95.0_{\pm 3.3}$ & $70.8_{\pm 6.4}$ & $94.9_{\pm 4.6}$ & $97.7_{\pm 3.5}$ & $\mathbf{94.1}_{\pm 4.8}$ & $\mathbf{94.7}_{\pm 4.7}$ & $91.8_{\pm 3.5}$ & $88.2_{\pm 2.5}$ & $74.2_{\pm 3.7}$ \\
\rowcolor{blue!8}
GiGPO + BATON (Ours) & $\mathbf{97.2}_{\pm 1.8}$ & $\mathbf{83.5}_{\pm 4.6}$ & $\mathbf{98.0}_{\pm 1.5}$ & $\mathbf{99.8}_{\pm 0.4}$ & $94.0_{\pm 3.1}$ & $85.6_{\pm 4.2}$ & $\mathbf{94.1}_{\pm 1.7}$ & $\mathbf{90.3}_{\pm 1.6}$ & $\mathbf{76.4}_{\pm 2.7}$ \\
\midrule
\multicolumn{10}{l}{\textit{Qwen2.5-7B-Instruct}} \\
Qwen2.5 & 33.7 & 21.4 & 19.7 & 6.6 & 3.0 & 3.1 & 14.9 & 26.1 & 8.0 \\
ReAct & 48.2 & 35.8 & 34.1 & 13.5 & 17.9 & 17.8 & 31.3 & 46.5 & 19.3 \\
Reflexion & 62.4 & 41.3 & 45.1 & 30.5 & 36.6 & 23.6 & 42.8 & 57.9 & 29.1 \\
PPO (with critic) & $92.5_{\pm 3.8}$ & $63.7_{\pm 8.6}$ & $92.8_{\pm 2.5}$ & $89.2_{\pm 6.8}$ & $80.7_{\pm 2.2}$ & $68.6_{\pm 8.1}$ & $80.5_{\pm 2.8}$ & $81.1_{\pm 3.3}$ & $68.9_{\pm 4.9}$ \\
RLOO & $87.3_{\pm 4.5}$ & $78.6_{\pm 8.1}$ & $87.1_{\pm 6.0}$ & $81.6_{\pm 7.4}$ & $71.5_{\pm 5.3}$ & $49.1_{\pm 8.2}$ & $75.4_{\pm 4.8}$ & $80.5_{\pm 3.0}$ & $65.4_{\pm 4.2}$ \\
GRPO & $91.1_{\pm 4.9}$ & $65.7_{\pm 6.9}$ & $89.5_{\pm 5.2}$ & $74.4_{\pm 7.1}$ & $72.9_{\pm 5.5}$ & $64.5_{\pm 7.1}$ & $77.7_{\pm 5.0}$ & $79.0_{\pm 2.9}$ & $66.3_{\pm 3.5}$ \\
GRPO + EnvRL & $92.6_{\pm 4.3}$ & $\mathbf{77.5}_{\pm 5.4}$ & $\mathbf{92.9}_{\pm 3.9}$ & $\mathbf{79.7}_{\pm 3.0}$ & $77.4_{\pm 3.7}$ & $71.1_{\pm 6.2}$ & $80.4_{\pm 3.9}$ & $81.8_{\pm 1.9}$ & $68.6_{\pm 2.3}$ \\
\rowcolor{blue!8}
GRPO + BATON (Ours) & $\mathbf{95.6}_{\pm 2.1}$ & $72.4_{\pm 5.3}$ & $90.3_{\pm 3.8}$ & $78.6_{\pm 4.9}$ & $\mathbf{94.1}_{\pm 3.2}$ & $\mathbf{75.3}_{\pm 4.6}$ & $\mathbf{86.9}_{\pm 2.7}$ & $\mathbf{83.4}_{\pm 1.9}$ & $\mathbf{70.3}_{\pm 2.5}$ \\
GiGPO & $97.5_{\pm 1.7}$ & $83.0_{\pm 7.7}$ & $98.6_{\pm 1.5}$ & $84.1_{\pm 7.4}$ & $89.0_{\pm 8.0}$ & $79.4_{\pm 6.8}$ & $90.9_{\pm 1.4}$ & $86.4_{\pm 2.4}$ & $74.9_{\pm 4.0}$ \\
GiGPO + EnvRL & $98.3_{\pm 1.5}$ & $92.2_{\pm 4.5}$ & $98.9_{\pm 0.8}$ & $93.8_{\pm 6.9}$ & $93.3_{\pm 3.5}$ & $\mathbf{94.6}_{\pm 3.0}$ & $94.5_{\pm 4.0}$ & $88.4_{\pm 4.9}$ & $76.3_{\pm 4.7}$ \\
\rowcolor{blue!8}
GiGPO + BATON (Ours) & $\mathbf{99.8}_{\pm 0.4}$ & $\mathbf{92.5}_{\pm 3.6}$ & $\mathbf{99.8}_{\pm 0.4}$ & $\mathbf{94.5}_{\pm 3.1}$ & $\mathbf{94.9}_{\pm 2.8}$ & $88.3_{\pm 3.9}$ & $\mathbf{94.9}_{\pm 1.5}$ & $\mathbf{89.8}_{\pm 1.8}$ & $\mathbf{79.5}_{\pm 2.6}$ \\
\bottomrule
\end{tabular}
}
\end{table*}

\section{Experiments}
\label{sec:experiments}

We evaluate the proposed dual-axis framework from three perspectives:
overall task performance, the independent and joint contributions of its two
optimization axes, and the execution behavior of the learned policies.

\subsection{Experimental Setup}
\label{sec:exp_setup}

We evaluate Qwen2.5-Instruct models at 1.5B and 7B scales on ALFWorld and
WebShop, and at 3B and 7B scales on search-augmented question answering.
We report success rate on ALFWorld, task score and success rate on WebShop,
and exact match on seven search-augmented QA datasets together with their
unweighted average.
We instantiate the framework with both GRPO and GiGPO as host learners and
compare against prompting and reinforcement-learning baselines, including
EnvRL. Matched variants share the same initial checkpoint, rollout budget,
reward and advantage construction, optimization settings, and evaluation
protocol, differing only in the component under study.



\subsection{Main Results}
\label{sec:exp_main}

Table~\ref{tab:main_results_agent} shows that the full dual-axis framework
consistently improves both GRPO and GiGPO across ALFWorld and WebShop and at
both model scales.
With Qwen2.5-1.5B-Instruct, the ALFWorld success rate increases from
72.9\% to 85.6\% for GRPO and from 86.8\% to 94.1\% for GiGPO.
On WebShop, the corresponding success rates improve from 57.1\% to 68.2\%
and from 67.1\% to 76.4\%, respectively.
The gains remain at the 7B scale: the framework improves ALFWorld from
77.7\% to 86.9\% with GRPO and from 90.9\% to 94.9\% with GiGPO, while
also improving their WebShop success rates from 66.3\% to 70.3\% and
from 74.9\% to 79.5\%.
These results indicate that jointly improving feedback attribution within
trajectories and objective aggregation across trajectories benefits host
learners with substantially different local credit constructions.
In particular, the improvements over GiGPO suggest that inter-trajectory
calibration and feedback-grounded supervision remain useful even when the
host learner already employs finer-grained within-trajectory credit signals.

\subsection{Search-Augmented Question Answering}
\label{sec:exp_searchqa}

We further evaluate BATON on search-augmented question answering, where
agents interact with external retrieval systems during multi-step reasoning.
Table~\ref{tab:searchqa_summary} reports the averaged exact match over seven
QA domains.
BATON improves both GRPO and GiGPO across model scales.
Compared with the corresponding host learners, BATON improves GRPO from
37.9\% to 41.1\% on Qwen2.5-3B and from 42.8\% to 45.9\% on Qwen2.5-7B.
Similarly, BATON improves GiGPO from 42.0\% to 43.2\% and from 47.1\% to
48.7\%, respectively.
The results indicate that the proposed dual-axis optimization remains
effective in retrieval-augmented agent settings.
Complete domain-level results are reported in Appendix~\ref{app:searchqa}.
\newcommand{\sqa}[2]{\ensuremath{\displaystyle #1_{\scriptstyle\pm #2}}}

\begin{table*}[t]
\centering
\caption{
Main results on search-augmented question answering.
We report exact match (\%) on seven evaluation domains and their unweighted
average.
NQ, TriviaQA, and PopQA are single-hop domains; HotpotQA,
2WikiMultihopQA, MuSiQue, and Bamboogle are multi-hop domains.
Within each backbone, bold marks the highest mean in each metric column
separately within the GRPO-based and GiGPO-based method groups.
}
\label{tab:searchqa_summary}
\footnotesize
\setlength{\tabcolsep}{1.35mm}
\renewcommand{\arraystretch}{1.12}
\resizebox{\linewidth}{!}{%
\begin{tabular}{lcccccccc}
\toprule
\multirow{2}{*}{Method}
& \multicolumn{3}{c}{Single-Hop QA}
& \multicolumn{4}{c}{Multi-Hop QA}
& \multicolumn{1}{c}{Overall} \\
\cmidrule(lr){2-4}\cmidrule(lr){5-8}\cmidrule(lr){9-9}
& NQ & TriviaQA & PopQA & HotpotQA & 2Wiki & MuSiQue & Bamboogle & Avg. \\
\midrule
\multicolumn{9}{l}{\textit{Qwen2.5-3B-Instruct}} \\

\addlinespace[2pt]
GRPO & \sqa{43.4}{2.8} & \sqa{62.1}{6.4} & \sqa{41.4}{5.1} & \sqa{37.8}{7.2} & \sqa{30.7}{4.6} & \sqa{14.8}{5.8} & \sqa{35.4}{3.7} & \sqa{37.9}{3.1} \\

GRPO + EnvRL & \sqa{44.6}{2.4} & \sqa{61.6}{6.9} & \sqa{42.3}{5.8} & \sqa{36.7}{5.3} & \sqa{31.1}{4.1} & \sqa{14.5}{4.9} & \sqa{39.0}{4.5} & \sqa{38.3}{3.6} \\

\rowcolor{blue!8}
GRPO + BATON (Ours) & \sqa{\mathbf{46.2}}{2.1} & \sqa{\mathbf{63.3}}{5.2} & \sqa{\mathbf{43.8}}{4.7} & \sqa{\mathbf{40.2}}{4.4} & \sqa{\mathbf{35.1}}{3.8} & \sqa{\mathbf{16.0}}{5.6} & \sqa{\mathbf{43.4}}{3.1} & \sqa{\mathbf{41.1}}{2.5} \\

\addlinespace[2pt]

GiGPO & \sqa{42.2}{5.7} & \sqa{59.3}{4.8} & \sqa{42.6}{3.9} & \sqa{36.7}{7.6} & \sqa{36.8}{4.9} & \sqa{12.8}{5.3} & \sqa{63.8}{6.8} & \sqa{42.0}{4.7} \\

GiGPO + EnvRL & \sqa{\mathbf{47.3}}{3.3} & \sqa{\mathbf{63.1}}{5.9} & \sqa{\mathbf{44.2}}{4.6} & \sqa{\mathbf{42.0}}{3.5} & \sqa{37.8}{4.8} & \sqa{\mathbf{17.6}}{4.7} & \sqa{44.2}{5.2} & \sqa{42.3}{3.4} \\

\rowcolor{blue!8}
GiGPO + BATON (Ours) & \sqa{42.9}{4.1} & \sqa{60.6}{6.1} & \sqa{44.1}{5.4} & \sqa{37.4}{5.8} & \sqa{\mathbf{39.9}}{4.2} & \sqa{13.3}{4.9} & \sqa{\mathbf{64.3}}{6.3} & \sqa{\mathbf{43.2}}{4.5} \\

\midrule

\multicolumn{9}{l}{\textit{Qwen2.5-7B-Instruct}} \\

\addlinespace[2pt]
GRPO & \sqa{47.7}{4.9} & \sqa{64.5}{6.9} & \sqa{45.5}{5.2} & \sqa{42.8}{7.1} & \sqa{38.6}{5.5} & \sqa{18.6}{6.2} & \sqa{42.2}{4.8} & \sqa{42.8}{3.6} \\

GRPO + EnvRL & \sqa{47.2}{4.3} & \sqa{64.8}{5.4} & \sqa{45.0}{3.9} & \sqa{\mathbf{44.2}}{3.0} & \sqa{37.3}{3.7} & \sqa{17.6}{6.2} & \sqa{42.6}{3.9} & \sqa{42.6}{2.8} \\

\rowcolor{blue!8}
GRPO + BATON (Ours) & \sqa{\mathbf{48.4}}{3.6} & \sqa{\mathbf{66.5}}{4.9} & \sqa{\mathbf{46.3}}{3.4} & \sqa{43.7}{3.8} & \sqa{\mathbf{52.4}}{5.5} & \sqa{\mathbf{19.7}}{4.7} & \sqa{\mathbf{44.1}}{2.8} & \sqa{\mathbf{45.9}}{2.7} \\

\addlinespace[2pt]

GiGPO & \sqa{46.2}{1.7} & \sqa{64.9}{7.7} & \sqa{46.3}{1.5} & \sqa{41.4}{7.4} & \sqa{43.8}{8.0} & \sqa{18.7}{6.8} & \sqa{68.7}{6.4} & \sqa{47.1}{4.0} \\

GiGPO + EnvRL & \sqa{45.9}{2.5} & \sqa{65.1}{4.5} & \sqa{47.2}{2.8} & \sqa{42.1}{6.9} & \sqa{42.7}{3.5} & \sqa{19.4}{3.0} & \sqa{69.3}{4.0} & \sqa{47.5}{3.7} \\

\rowcolor{blue!8}
GiGPO + BATON (Ours) & \sqa{\mathbf{47.8}}{1.8} & \sqa{\mathbf{65.6}}{3.6} & \sqa{\mathbf{49.5}}{1.9} & \sqa{\mathbf{43.3}}{3.1} & \sqa{\mathbf{45.0}}{3.8} & \sqa{\mathbf{20.4}}{4.5} & \sqa{\mathbf{69.4}}{2.7} & \sqa{\mathbf{48.7}}{2.6} \\
\bottomrule
\end{tabular}%
}
\end{table*}



\subsection{Component Ablation}
\label{sec:component_ablation}

Table~\ref{tab:component_ablation} provides a controlled $2\times2$
decomposition of BATON: the host learner alone, Bayesian Feedback Attribution
only, TMN only, and their combination.
Both axes yield substantial gains independently.
On ALFWorld with Qwen2.5-1.5B-Instruct, GRPO improves from 72.9\% to
84.1\% with Bayesian Feedback Attribution and to 82.4\% with TMN, while
BATON reaches 85.6\%.
For GiGPO, the corresponding results are 86.8\%, 93.4\%, 90.7\%, and 94.1\%.
The same complementarity holds across tasks and scales.
On WebShop with the 7B backbone, BATON reaches 70.3\% success with GRPO
and 79.5\% with GiGPO, outperforming either component alone.
These results support the dual-axis formulation: the two components are
independently effective and provide complementary gains when combined.
\providecommand{\ablval}[2]{%
  \ensuremath{\displaystyle #1_{\scriptstyle\pm #2}}}

\begin{table*}[t]
\centering
\caption{
Component ablation of BATON's two optimization axes.
BFA denotes Bayesian Feedback Attribution for Intra-Trajectory Feedback
Attribution, while TMN denotes Trajectory Mass Normalization for
Inter-Trajectory Objective Aggregation.
Values are reported as means with standard deviations.
}
\label{tab:component_ablation}
\footnotesize
\setlength{\tabcolsep}{1.45mm}
\renewcommand{\arraystretch}{1.12}
\resizebox{\linewidth}{!}{%
\begin{tabular}{lcccccccc}
\toprule
\multirow{3}{*}{Component setting}
& \multicolumn{2}{c}{ALFWorld}
& \multicolumn{4}{c}{WebShop}
& \multicolumn{2}{c}{Search-augmented QA} \\
\cmidrule(lr){2-3}\cmidrule(lr){4-7}\cmidrule(lr){8-9}
& 1.5B & 7B
& \multicolumn{2}{c}{1.5B}
& \multicolumn{2}{c}{7B}
& 3B & 7B \\
\cmidrule(lr){4-5}\cmidrule(lr){6-7}
& All & All & Score & Succ. & Score & Succ. & Avg. & Avg. \\
\midrule
GRPO
& \ablval{72.9}{3.4} & \ablval{77.7}{5.0} & \ablval{75.6}{3.6} & \ablval{57.1}{3.6} & \ablval{79.0}{2.9} & \ablval{66.3}{3.5} & \ablval{37.9}{3.1} & \ablval{42.8}{3.6} \\
+ BFA
& \ablval{{84.1}}{0.7} & \ablval{{85.4}}{3.2} & \ablval{{82.1}}{1.8} & \ablval{{67.3}}{2.9} & \ablval{{82.8}}{1.3} & \ablval{{69.4}}{1.8} & \ablval{{40.7}}{2.0} & \ablval{{44.2}}{2.2} \\

+ TMN
& \ablval{{82.4}}{2.9} & \ablval{{83.8}}{3.0} & \ablval{{79.6}}{2.5} & \ablval{{64.0}}{3.2} & \ablval{{81.5}}{2.3} & \ablval{{67.6}}{2.9} & \ablval{{39.3}}{2.8} & \ablval{{43.0}}{2.9} \\
\rowcolor{blue!8}
+ BATON (Full)
& \ablval{\mathbf{85.6}}{2.8} & \ablval{\mathbf{86.9}}{2.7} & \ablval{\mathbf{85.2}}{2.0} & \ablval{\mathbf{68.2}}{2.8} & \ablval{\mathbf{83.4}}{1.9} & \ablval{\mathbf{70.3}}{2.5} & \ablval{\mathbf{41.1}}{2.5} & \ablval{\mathbf{45.9}}{2.7} \\
\midrule
GIGPO
& \ablval{86.8}{1.8} & \ablval{90.9}{1.4} & \ablval{83.7}{1.9} & \ablval{67.1}{4.7} & \ablval{86.4}{2.4} & \ablval{74.9}{4.0} & \ablval{42.0}{4.7} & \ablval{47.1}{4.0} \\
+ BFA
& \ablval{{93.4}}{1.4} & \ablval{{93.6}}{2.7} & \ablval{{88.3}}{1.5} & \ablval{{73.4}}{2.1} & \ablval{{89.6}}{1.1} & \ablval{{78.6}}{3.7} & \ablval{{42.3}}{3.9} & \ablval{{48.2}}{3.4} \\

+ TMN
& \ablval{{90.7}}{1.9} & \ablval{{91.3}}{1.6} & \ablval{{86.0}}{1.8} & \ablval{71.6}{3.1} & \ablval{87.7}{2.1} & \ablval{{76.3}}{2.8} & \ablval{42.9}{3.1} & \ablval{47.4}{2.8} \\
\rowcolor{blue!8}
+ BATON (Full)
& \ablval{\mathbf{94.1}}{1.7} & \ablval{\mathbf{94.9}}{1.5} & \ablval{\mathbf{90.3}}{1.6} & \ablval{\mathbf{76.4}}{2.7} & \ablval{\mathbf{89.8}}{1.8} & \ablval{\mathbf{79.5}}{2.6} & \ablval{\mathbf{43.2}}{4.5} & \ablval{\mathbf{48.7}}{2.6} \\
\bottomrule
\end{tabular}%
}
\end{table*}

\subsection{Trajectory Execution Cost}
\label{sec:exp_trajectory_cost}

We further examine whether the full dual-axis framework changes the execution
behavior of the learned policies on ALFWorld with Qwen2.5-1.5B-Instruct.
Figure~\ref{fig:exp_trajectory_cost} reports the mean number of environment
actions and generated policy tokens per evaluation trajectory, including both
successful and failed episodes.
For GRPO, the full framework reduces the mean action count from 21.51 to
16.39 and the mean generated-token count from 1,734 to 1,301.
For GiGPO, the corresponding quantities decrease from 13.51 to 10.82 actions
and from 1,431 to 746 generated tokens.
These results indicate that the performance improvements are accompanied by
shorter and less verbose interaction trajectories, suggesting that the learned
policies make more efficient use of their interaction budget.
\paragraph{Computational overhead.}
BATON introduces no additional environment interaction.
The additional computation comes from Bayesian Feedback Attribution,
which performs feedback likelihood evaluation for attribution weighting,
while TMN only modifies the trajectory-level aggregation rule.
The wall-clock overhead analysis is provided in Appendix~\ref{app:overhead}.

\begin{figure}[t]
    \centering
    \includegraphics[
    width=\linewidth,
    trim=0 20 0 20 ,
    clip
]
        {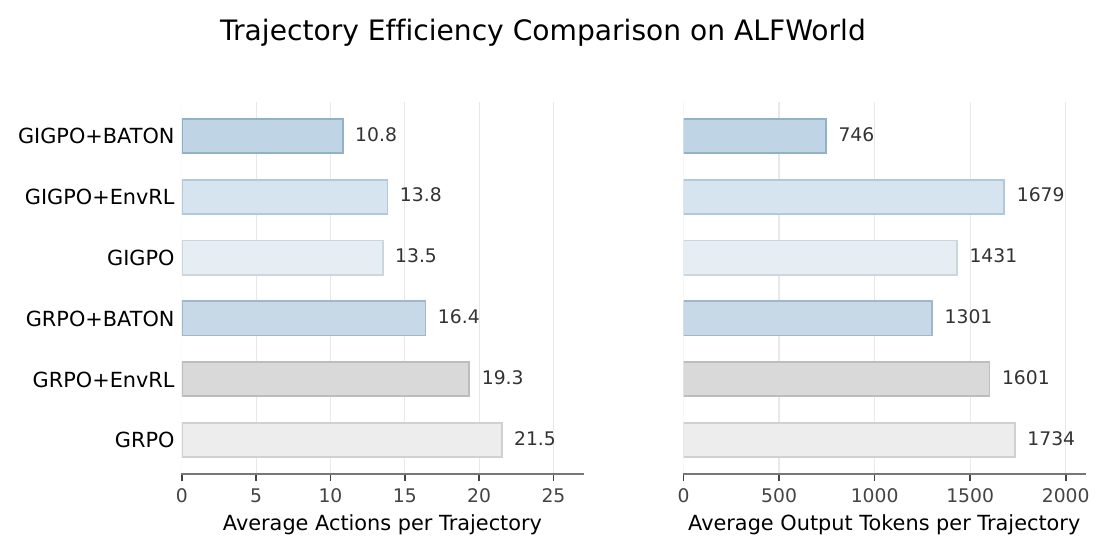}
    \caption{
    Trajectory execution cost on ALFWorld with Qwen2.5-1.5B-Instruct.
    We report mean environment-action counts and generated policy-token counts
    for the host learners, EnvRL variants, and the full dual-axis framework at
    update 150, including both successful and failed episodes.
    Lower is better.
    }
    \label{fig:exp_trajectory_cost}
\end{figure}

\section{Related Work}
\label{sec:related_work}

\paragraph{Policy optimization for LLMs and agents.}
PPO-based RLHF provides the foundation for LLM policy optimization
\citep{schulman2017ppo,ouyang2022training}, while RLOO and GRPO introduce
critic-free and group-relative alternatives
\citep{ahmadian2024back,grpo}.
Recent reasoning-oriented methods improve optimization through modified
clipping, reduction, process supervision, and sequence-level objectives,
including DAPO, Dr.~GRPO, REINFORCE++, PRIME, Open-Reasoner-Zero,
SimpleRL-Zoo, GSPO, MGAP, TASPO and SARE
\citep{yu2025dapo,liu2025drgrpo,hu2025reinforcepp,cui2025prime,
hu2025openreasonerzero,zeng2025simplerlzoo,zheng2025gspo,yang2026reconcilingprocesssupervisionoutcomebased,yang2026saresamplewiseadaptivereasoning,zhuang2026mitigatingmanifolddepartureuncertaintyaware}.
For interactive agents, WebRL and RAGEN study reinforcement learning in
environmental settings
\citep{qi2024webrl,wang2025ragen}, while GiGPO, turn-level credit assignment,
and Agent Lightning investigate finer-grained trajectory optimization
\citep{gigpo,zeng2025turncredit,luo2025agentlightning}.
BATON is complementary to these approaches by separating within-trajectory
feedback attribution from across-trajectory objective aggregation while
preserving the host learner's optimization structure.

\paragraph{Learning from interaction feedback.}
Interaction feedback has long been used to ground language agents through
actions and subsequent observations
\citep{yao2023react,schick2023toolformer}.
Reflexion, Self-Refine, CRITIC, LATS, and Voyager exploit verbal, tool-based,
or environmental feedback for iterative improvement
\citep{shinn2023reflexion,madaan2023selfrefine,gou2024critic,
zhou2023lats,wang2023voyager}.
AgentTuning and AgentGym further study learning from interaction trajectory
collections
\citep{zeng2023agenttuning,xi2024agentgym}, while EnvRL introduces
state-prediction and inverse-dynamics objectives for agent reinforcement
learning
\citep{envrl}.
Different from these approaches, Bayesian Feedback Attribution compares the
observed feedback compatibility of executed and policy-sampled alternative
actions, and uses the resulting evidence to reweight the host policy objective
without replacing its reward or credit estimator.

\paragraph{Interactive agent environments and search.}
Long-horizon agent learning has been studied in embodied and text environments
including ALFWorld, WebShop, and ScienceWorld
\citep{shridhar2021alfworld,yao2022webshop,wang2022scienceworld},
as well as broader evaluation suites such as AgentBench, AgentBoard, WebArena,
Mind2Web, WorkArena, and WebLINX
\citep{liu2023agentbench,ma2024agentboard,zhou2023webarena,
deng2023mind2web,drouin2024workarena,lu2024weblinx}.
Recent tool-use benchmarks and software environments further include
$\tau$-bench and SWE-agent
\citep{yao2024taubench,yang2024sweagent}.
Search-R1, ZeroSearch, and StepSearch optimize multi-turn retrieval agents
with reinforcement learning
\citep{jin2025searchr1,sun2025zerosearch,zheng2025stepsearch}.
These environments motivate explicit treatment of long and heterogeneous
interaction trajectories, which BATON addresses through separate attribution
and aggregation mechanisms.

\section{Conclusion}
\label{sec:conclusion}

We present \textbf{BATON}, a dual-axis policy optimization framework for
multi-turn LLM agents.
BATON combines Bayesian Feedback Attribution, which reallocates learning mass
among decisions within trajectories based on feedback compatibility, with
Trajectory Mass Normalization, which balances contributions across trajectories.
The framework preserves the host learner's optimization structure and is
compatible with GRPO and GiGPO.
Experiments on ALFWorld, WebShop, and search-augmented question answering
demonstrate consistent improvements across model scales.



\bibliography{iclr2027_conference}
\bibliographystyle{iclr2027_conference}

\appendix
\section{Proofs of Theoretical Properties}
\label{app:bfa_theory}

\subsection{Proof of Proposition~\ref{prop:evidence_interpretation}}

\begin{proof}

Under the pairwise construction, the Bayesian posterior is

\[
\rho_{\theta_{\rm old}}(a|h,o')
=
\frac{
f_{\theta_{\rm old}}(o'|h,a)
}{
f_{\theta_{\rm old}}(o'|h,a)
+
f_{\theta_{\rm old}}(o'|h,\tilde a)
}.
\]

Since the empirical prior is uniform,

\[
\rho_0(a|h)=\frac12 .
\]

Therefore,

\[
e(h,a,o')
=
\frac{
\rho_{\theta_{\rm old}}(a|h,o')
}{
\rho_0(a|h)
}
\]

can be written as

\[
e(h,a,o')
=
\frac{
2f_{\theta_{\rm old}}(o'|h,a)
}{
f_{\theta_{\rm old}}(o'|h,a)
+
f_{\theta_{\rm old}}(o'|h,\tilde a)
}.
\]

Let

\[
m(o'|h)
=
\frac12f_{\theta_{\rm old}}(o'|h,a)
+
\frac12f_{\theta_{\rm old}}(o'|h,\tilde a).
\]

Then,

\[
e(h,a,o')
=
\frac{
f_{\theta_{\rm old}}(o'|h,a)
}{
m(o'|h)
}.
\]

Taking expectation over
$o'\sim f_{\theta_{\rm old}}(\cdot|h,a)$ gives

\[
\begin{aligned}
&
\mathbb E_{o'\sim f_{\theta_{\rm old}}(\cdot|h,a)}
[
\log e(h,a,o')
]
\\
&=
\mathbb E_{o'\sim f_{\theta_{\rm old}}(\cdot|h,a)}
\left[
\log
\frac{
f_{\theta_{\rm old}}(o'|h,a)
}{
m(o'|h)
}
\right]
\\
&=
D_{\mathrm{KL}}
(
f_{\theta_{\rm old}}(\cdot|h,a)
\|
m(\cdot|h)
)
\geq0 .
\end{aligned}
\]

Thus, the expected log evidence equals the KL divergence between the
executed action's feedback distribution and the pairwise mixture distribution.

\end{proof}

\section{Computational Overhead Analysis}
\label{app:overhead}

We analyze the additional computational cost introduced by BATON.
All measurements are conducted on ALFWorld with Qwen2.5-1.5B-Instruct
under the same rollout budget, batch size, optimizer settings, and hardware
configuration as the main experiments.

BATON does not introduce additional environment interaction.
The overhead mainly comes from Bayesian Feedback Attribution, which requires
additional feedback likelihood evaluation for attribution weighting.
Trajectory Mass Normalization only changes the trajectory-level aggregation
rule and therefore introduces negligible computational cost.

Table~\ref{tab:overhead} reports the total wall-clock training time and the
relative cost normalized by the corresponding host learner.

\section{Additional Search-Augmented QA Results}
\label{app:searchqa}

Table~\ref{tab:main_results_searchqa} reports the domain-level exact match results
on the seven search-augmented question answering benchmarks.
The averaged score reported in the main paper is computed as the unweighted
mean over these domains.

\begin{table*}[hbtp]
\centering
\caption{
Main results on search-augmented question answering.
We report exact match (\%) on seven evaluation domains and their unweighted
average.
NQ, TriviaQA, and PopQA are single-hop domains; HotpotQA,
2WikiMultihopQA, MuSiQue, and Bamboogle are multi-hop domains.
Within each backbone, bold marks the highest mean in each metric column
separately within the GRPO-based and GiGPO-based method groups.
BATON rows are shaded in blue.
}
\label{tab:main_results_searchqa}
\footnotesize
\setlength{\tabcolsep}{1.35mm}
\renewcommand{\arraystretch}{1.12}
\resizebox{\linewidth}{!}{%
\begin{tabular}{lcccccccc}
\toprule
\multirow{2}{*}{Method}
& \multicolumn{3}{c}{Single-Hop QA}
& \multicolumn{4}{c}{Multi-Hop QA}
& \multicolumn{1}{c}{Overall} \\
\cmidrule(lr){2-4}\cmidrule(lr){5-8}\cmidrule(lr){9-9}
& NQ & TriviaQA & PopQA & HotpotQA & 2Wiki & MuSiQue & Bamboogle & Avg. \\
\midrule
\multicolumn{9}{l}{\textit{Qwen2.5-3B-Instruct}} \\
R1-Instruct & \sqa{26.8}{3.8} & \sqa{53.9}{6.2} & \sqa{19.7}{4.5} & \sqa{23.9}{5.7} & \sqa{29.0}{4.1} & \sqa{7.4}{2.8} & \sqa{29.1}{5.4} & \sqa{27.0}{3.6} \\
Search-R1 & \sqa{34.3}{2.9} & \sqa{54.2}{5.8} & \sqa{38.0}{4.7} & \sqa{32.2}{6.3} & \sqa{32.1}{4.9} & \sqa{10.1}{3.6} & \sqa{26.6}{5.1} & \sqa{32.5}{3.8} \\
ZeroSearch & \sqa{41.1}{2.6} & \sqa{57.6}{4.9} & \sqa{44.6}{5.6} & \sqa{27.6}{5.2} & \sqa{29.8}{4.3} & \sqa{9.9}{3.9} & \sqa{11.3}{2.7} & \sqa{31.7}{3.4} \\
StepSearch & -- & -- & -- & \sqa{34.4}{5.8} & \sqa{32.2}{4.6} & \sqa{17.6}{4.2} & \sqa{34.2}{6.1} & -- \\
\addlinespace[2pt]
GRPO & \sqa{43.4}{2.8} & \sqa{62.1}{6.4} & \sqa{41.4}{5.1} & \sqa{37.8}{7.2} & \sqa{30.7}{4.6} & \sqa{14.8}{5.8} & \sqa{35.4}{3.7} & \sqa{37.9}{3.1} \\
GRPO + EnvRL & \sqa{44.6}{2.4} & \sqa{61.6}{6.9} & \sqa{42.3}{5.8} & \sqa{36.7}{5.3} & \sqa{31.1}{4.1} & \sqa{14.5}{4.9} & \sqa{39.0}{4.5} & \sqa{38.3}{3.6} \\
\rowcolor{blue!8}
GRPO + BATON (Ours) & \sqa{\mathbf{46.2}}{2.1} & \sqa{\mathbf{63.3}}{5.2} & \sqa{\mathbf{43.8}}{4.7} & \sqa{\mathbf{40.2}}{4.4} & \sqa{\mathbf{35.1}}{3.8} & \sqa{\mathbf{16.0}}{5.6} & \sqa{\mathbf{43.4}}{3.1} & \sqa{\mathbf{41.1}}{2.5} \\
\addlinespace[2pt]
GiGPO & \sqa{42.2}{5.7} & \sqa{59.3}{4.8} & \sqa{42.6}{3.9} & \sqa{36.7}{7.6} & \sqa{36.8}{4.9} & \sqa{12.8}{5.3} & \sqa{63.8}{6.8} & \sqa{42.0}{4.7} \\
GiGPO + EnvRL & \sqa{\mathbf{47.3}}{3.3} & \sqa{\mathbf{63.1}}{5.9} & \sqa{\mathbf{44.2}}{4.6} & \sqa{\mathbf{42.0}}{3.5} & \sqa{37.8}{4.8} & \sqa{\mathbf{17.6}}{4.7} & \sqa{44.2}{5.2} & \sqa{42.3}{3.4} \\
\rowcolor{blue!8}
GiGPO +  BATON (Ours) & \sqa{42.9}{4.1} & \sqa{60.6}{6.1} & \sqa{44.1}{5.4} & \sqa{37.4}{5.8} & \sqa{\mathbf{39.9}}{4.2} & \sqa{13.3}{4.9} & \sqa{\mathbf{64.3}}{6.3} & \sqa{\mathbf{43.2}}{4.5} \\
\midrule
\multicolumn{9}{l}{\textit{Qwen2.5-7B-Instruct}} \\
R1-Instruct & \sqa{20.8}{4.2} & \sqa{45.1}{7.1} & \sqa{16.9}{3.8} & \sqa{21.0}{5.4} & \sqa{27.3}{4.7} & \sqa{6.2}{2.5} & \sqa{19.0}{4.9} & \sqa{22.3}{3.9} \\
Search-R1 & \sqa{39.5}{3.6} & \sqa{60.8}{5.4} & \sqa{39.9}{4.2} & \sqa{36.8}{6.8} & \sqa{40.3}{5.9} & \sqa{14.4}{4.1} & \sqa{37.0}{5.3} & \sqa{38.6}{3.7} \\
ZeroSearch & \sqa{43.8}{2.9} & \sqa{61.6}{6.2} & \sqa{51.7}{5.7} & \sqa{34.4}{4.8} & \sqa{35.4}{4.4} & \sqa{18.2}{5.6} & \sqa{27.6}{4.6} & \sqa{39.2}{3.5} \\
StepSearch & -- & -- & -- & \sqa{38.8}{5.1} & \sqa{36.4}{4.7} & \sqa{22.8}{5.3} & \sqa{39.8}{6.5} & -- \\
\addlinespace[2pt]
GRPO & \sqa{47.7}{4.9} & \sqa{64.5}{6.9} & \sqa{45.5}{5.2} & \sqa{42.8}{7.1} & \sqa{38.6}{5.5} & \sqa{18.6}{6.2} & \sqa{42.2}{4.8} & \sqa{42.8}{3.6} \\
GRPO + EnvRL & \sqa{47.2}{4.3} & \sqa{64.8}{5.4} & \sqa{45.0}{3.9} & \sqa{\mathbf{44.2}}{3.0} & \sqa{37.3}{3.7} & \sqa{17.6}{6.2} & \sqa{42.6}{3.9} & \sqa{42.6}{2.8} \\
\rowcolor{blue!8}
GRPO + BATON (Ours) & \sqa{\mathbf{48.4}}{3.6} & \sqa{\mathbf{66.5}}{4.9} & \sqa{\mathbf{46.3}}{3.4} & \sqa{43.7}{3.8} & \sqa{\mathbf{52.4}}{5.5} & \sqa{\mathbf{19.7}}{4.7} & \sqa{\mathbf{44.1}}{2.8} & \sqa{\mathbf{45.9}}{2.7} \\
\addlinespace[2pt]
GiGPO & \sqa{46.2}{1.7} & \sqa{64.9}{7.7} & \sqa{46.3}{1.5} & \sqa{41.4}{7.4} & \sqa{43.8}{8.0} & \sqa{18.7}{6.8} & \sqa{68.7}{6.4} & \sqa{47.1}{4.0} \\
GiGPO + EnvRL & \sqa{45.9}{2.5} & \sqa{65.1}{4.5} & \sqa{47.2}{2.8} & \sqa{42.1}{6.9} & \sqa{42.7}{3.5} & \sqa{19.4}{3.0} & \sqa{69.3}{4.0} & \sqa{47.5}{3.7} \\
\rowcolor{blue!8}
GiGPO + BATON (Ours) & \sqa{\mathbf{47.8}}{1.8} & \sqa{\mathbf{65.6}}{3.6} & \sqa{\mathbf{49.5}}{1.9} & \sqa{\mathbf{43.3}}{3.1} & \sqa{\mathbf{45.0}}{3.8} & \sqa{\mathbf{20.4}}{4.5} & \sqa{\mathbf{69.4}}{2.7} & \sqa{\mathbf{48.7}}{2.6} \\
\bottomrule
\end{tabular}%
}
\end{table*}

\begin{table}[hbtp]
\centering
\caption{
Wall-clock training overhead of BATON on ALFWorld.
Training time is measured over the complete training process under identical
rollout and optimization settings. Values are normalized by the corresponding
host learner.
}
\label{tab:overhead}
\small
\begin{tabular}{lcc}
\toprule
Method & Training Time (hours) & Relative Cost \\
\midrule
GRPO
& 18.05 & 1.00$\times$ \\
GRPO + TMN
& 18.50 & 1.02$\times$ \\
GRPO + BATON
& 21.07 & 1.17$\times$ \\
\midrule
GiGPO
& 22.77 & 1.00$\times$ \\
GiGPO + TMN
& 23.08 & 1.01$\times$ \\
GiGPO + BATON
& 25.15 & 1.10$\times$ \\
\bottomrule
\end{tabular}
\end{table}
\end{document}